\documentclass[sigconf]{aamas} 
\usepackage{booktabs} 
\usepackage{balance}
\setcopyright{ifaamas}  
\acmDOI{}  
\acmISBN{}  
\acmConference[AAMAS'18]{Proc.\@ of the 17th International Conference on Autonomous Agents and Multiagent Systems (AAMAS 2018)}{July 10--15, 2018}{Stockholm, Sweden}{M.~Dastani, G.~Sukthankar, E.~Andr\'{e}, S.~Koenig (eds.)}  
\acmYear{2018}  
\copyrightyear{2018}  
\acmPrice{}  

\usepackage{graphicx}  
\usepackage{amsmath}

\newcommand{\sub}{_}
\def\su{^}

\newcommand{\know}{{\it Know}}
\newcommand{\bel}{{\it Bel}}

\newcommand{\lt}{<}
\newcommand{\gt}{>}

\newcommand{\A}{{\cal A}}

\newcommand{\D}{{\cal D}}

\newcommand{\cal}{\mathcal}

\newcommand{\ivar}{{\iota}}

\renewcommand{\L}{{\cal L}}

\newcommand{\N}{{\cal N}}
\renewcommand{\O}{{\cal O}}

\newcommand{\Q}{{\cal Q}}
\newcommand{\obs}{{\it SF}}

\newcommand{\fsa}{\mathit{\Sigma}}

\newcommand{\X}{{\cal X}}

\newcommand{\ti}{\textit}

\newcommand{\set}[1]{\left\{ #1 \right\}}

\newcommand{\la}{\langle}
\newcommand{\ra}{\rangle}

\newcommand{\eg}{\emph{e.g.}~}

\newcommand{\vs}{vs.~}

\newcommand{\qi}{{Q\sub 0}}
\newcommand{\qf}{{Q\sub F}}

\newcommand{\poss}{\ti{Poss}}

\newcommand{\ins}{{S_{0}}}

\newcommand{\thereis}{\exists}

\newcommand{\ts}{{T^*}}
\newcommand{\alt}{\textit{alt}}

\newcommand{\true}{\mbox{{\it true}}}

\begin{document}
\title{On Plans With Loops and Noise} 



\author{Vaishak Belle} 
\affiliation{%
  \institution{University of Edinburgh \& Alan Turing Institute}
}
\email{vaishak@ed.ac.uk}

\begin{abstract} In an influential paper, Levesque proposed a formal specification for analysing the correctness of program-like plans, such as conditional plans, iterative plans, and knowledge-based plans. He motivated a logical characterisation within the situation calculus that included binary sensing actions. While the characterisation does not immediately yield a practical algorithm, the specification serves as a general skeleton to explore the synthesis of program-like plans for reasonable, tractable fragments. 

Increasingly, classical plan structures are being applied to stochastic environments such as robotics applications. This raises the question as to what the specification for correctness should look like, since Levesque's account makes the assumption that sensing is exact and actions are deterministic. Building on a
situation calculus theory for reasoning about
degrees of belief and noise, we revisit the execution
semantics of generalised plans. The specification is then used to
analyse the correctness of example plans.

\end{abstract}

\maketitle

\section{Introduction} 
\label{sec:introduction}

In an influential paper, Levesque \cite{DBLP:conf/aaai/Levesque96}  proposed a formal specification for analysing the correctness of program-like plans, such as conditional plans, iterative plans, and knowledge-based plans. The problem setting is this: in a world where the agent can affect changes by acting and learn about the truth of fluent values by sensing, what should a plan \textit{look like}  and how should we verify that it is \textit{correct}? 
As a simple example, consider the problem of chopping down a tree of unknown thickness using a \textit{chop} action that reduces its thickness by a unit.  
Clearly, no fixed sequence of chops would work; however, if one is able to check after each chop whether the tree still stands, then a simple iterative plan like in {Figure}~\ref{fig:tc} achieves the desired outcome. To analyse such plans, Levesque motivated an epistemic  characterisation within the logical language of the situation calculus that included binary sensing actions. In that account, the planning task is to find a structure such that for every situation (that is, world state) considered initially possible, it leads to a final situation where the goal holds. While the characterisation does not immediately yield a practical algorithm, the specification serves as a general skeleton to explore the synthesis of program-like plans for reasonable, tractable fragments \cite{Cimatti200335,DBLP:conf/aips/BonetPG09,DBLP:conf/aips/HuG13,srivastava2015tractability}.  Moreover, recent advances in multi-agent epistemic planning \cite{DBLP:conf/aips/KominisG15,muise-aaai-15} are encouraging, and a formal characterisation like the one by Levesque can help relate that work to generalised planning.

\begin{figure}[t]
  \centering
     \includegraphics[width=.4\textwidth]{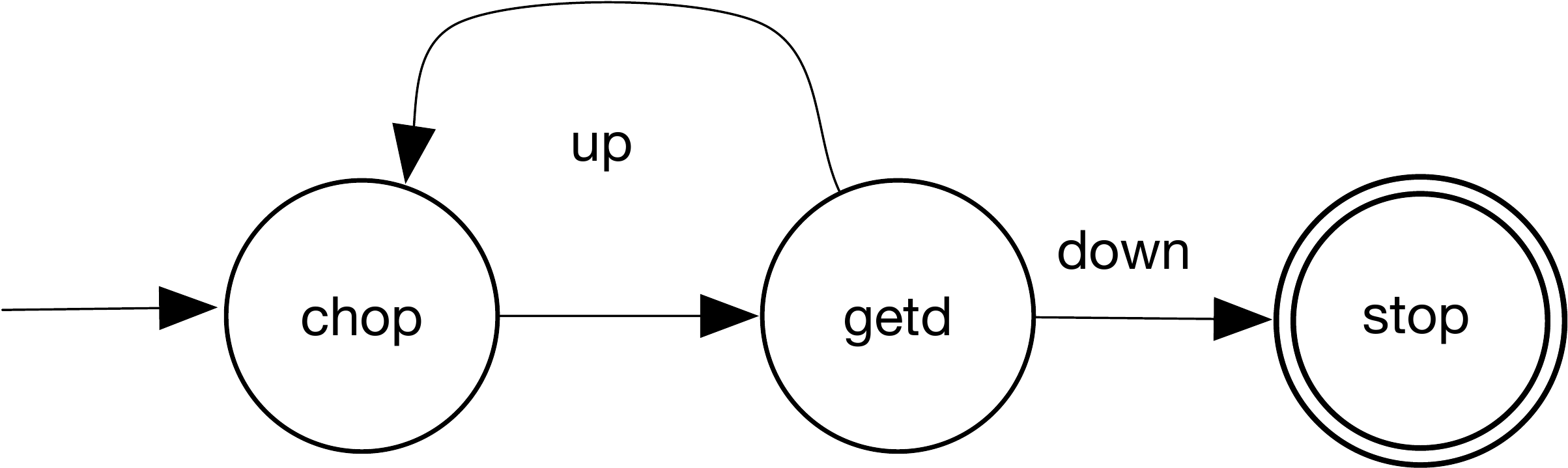}
  \caption{controller for chopping a tree of unknown non-zero thickness $d$, by means of a binary sensing action $getd$}
  \label{fig:tc}
\end{figure}

Increasingly, classical plan structures are being applied to stochastic environments such as robotics applications \cite{mataric2007robotics}. Indeed, uncertainty in the initial parameters of the planning problem is equivalent to reasoning about belief states, so robotics has long served as a motivation for generalised planning algorithms \cite{DBLP:conf/aips/BonetPG09}. Approaches such as \cite{srivastava2015tractability} further allow a degree of nondeterminism  in the effects of actions.  This raises the question as to what the specification for correctness should look like in general, since Levesque's account makes the assumption that sensing is exact and actions are deterministic. In fact, Levesque concludes his paper with: \smallskip 

{\it  ``But suppose that sensing involves reading
from a noisy sensor \( \ldots \) how robot
programs or planning could be defined in terms of
this account still remains to be seen."} \smallskip 

To the best of our knowledge, no attempt has been made to address the issue at the level of generality of the original paper, and this is precisely our aim here. Building on a
situation calculus theory for reasoning about
degrees of belief and noise, our main contribution is to revisit the execution
semantics of generalised plans. Concretely, beginning with the established case of exact sensing and deterministic acting, we turn to the case of exact sensing but noisy acting. We then motivate the case where both acting and sensing is noisy. We formally establish some compatibility theorems between these accounts.  Finally, the specification is then used to
analyse the correctness of example plans. 

We reiterate that the technical thrust of this paper is  limited to formal characterisations: at no point  will we concern ourselves with algorithmic ideas or plan heuristics. 
We will mainly show how the account correctly handles changes to the state of the  world as a result of noisy actions, as well as the changes to the beliefs of an agent after noisy sensing. 

In this paper, there is a natural evolution of theory and formulation when compared to Levesque's account. The original account was based on the situation calculus extended for knowledge and sensing \cite{citeulike:528170} derived from classical epistemic logic \cite{reasoning:about:knowledge}. Our account is based on the situation calculus extended for probabilistic belief and noise \cite{bacchus1999171} derived from probabilistic epistemic logic \cite{bacchus1990representing,174658}. 
{In principle, of course, any logical language  for reasoning about actions and probabilities could have been used for the formalisation, e.g. \cite{kooi2003probabilistic}. But by using the situation calculus, we can clearly explicate the generalisation from Levesque's account, but also benefit from its first-order expressiveness, at least for axiomatising the planning domain. Moreover, to logically characterise unbounded iteration and transitive closure, we use second-order logic, as would Levesque. 

Its worth remarking that while this logical machinery makes the account more involved than (say) POMDP specifications \cite{Kaelbling199899}, having a more general language is useful for contextualising involved extensions such as the handling of non-unique prior distributions. For example, in \cite{DBLP:journals/ijrr/KaelblingL13}, it is argued that when planning in highly stochastic and unknown environments, it is useful to allow a margin of error in what the values of fluents will be. This means that the planning system has to reason about multiple distributions satisfying such constraints, and as in \cite{DBLP:journals/ijrr/KaelblingL13}, the use of logical connectives allows us to express such scenarios effortlessly. 
} 

\renewcommand{\qed}{}

\section{A Theory of Knowledge and Action} 
\label{sec:background}

We will not go over the language \( {\mathcal L} \) of the situation calculus in detail \cite{McCarthy:69,reiter2001knowledge}, but simply note that it 
is a
many-sorted dialect of predicate calculus, with sorts for \emph{actions},
\emph{situations}, denoting a (possibly) empty sequence of actions, and \emph{objects}, for everything else. A special constant \( \ins \) denotes the real world initially, and the term \( do(a,s) \) denotes the situation obtained on
doing \( a \) in \( s. \) Fluents, whose last argument is always a situation, can be used to capture changing properties. Following \cite{reiter2001knowledge}, application domains are axiomatised as \emph{basic action theories}, which stipulate the conditions under which actions are executable, and their affects on fluents, while embodying a monotonic solution to the frame problem. For example, for the tree chop problem, using the functional fluent $d$ to mean the thickness of the tree, we may have:\footnote{Free variables are assumed to be implicitly quantified from the outside. For readability purposes, we let  \( \ivar \) range over initial situations only, that is, where no actions have occurred.}\\

$\begin{array}{l}
	\poss(chop(x),s) \equiv d(s)  \geq x. \\[2ex] 
	d(do(a,s)) = u ~\equiv \\ \quad \quad\quad\quad\quad \quad\quad\quad  (a = chop(x) \land d(s) = u+x) ~\lor \\  \quad \quad\quad\quad\quad  \quad\quad\quad
	 (a\neq chop(x) \land d(s) = u). 
\end{array}$ \\

We let $chop$ be an abbreviation for $chop(1)$. 
These axioms say that $chop$ is only possible when the tree's thickness is non-zero, and that the current value of $d$ is $u$ if and only if its previous value was also $u$ and no $chop$ action occurred, or its previous value was $u+1$ and a single $chop$ action was executed. 

A special binary fluent \( K(s',s) \) denotes that \( s' \) is a possible world when the agent is at \( s. \)  As usual, knowledge is defined as truth at accessible worlds:
\[
	\know(\phi,s) \doteq \forall s'.~ K(s',s) \supset \phi[s'].
\]
A modeler axiomatises the initial beliefs of the agent:\footnote{Like in modal logic \cite{reasoning:about:knowledge}, constraints on \( K \) correspond to appropriate properties for \( \know \) in the truth theory \cite{citeulike:528170}. We assume, as is usual, that \( \know \) has the full power of introspection and closed under logical reasoning -- so-called \textbf{S5} --  by consequence of a stipulation that \( K \) is an equivalence relation. Also note that, unlike standard modal logic, where ``worlds'' are static states of affairs that provide truth values to propositions, the model theory of the situation calculus instantiates \textit{trees}: each world describes the values of fluents initially, but also after any sequence of actions.}  
\begin{equation}\label{eq:knowledge thickness discrete}
	K(\ivar, \ins) \supset (1 \leq d(\ivar) \leq 10).
\end{equation} 
This is equivalently written using $\know$ and a special term $now$ to denote the current situation as: $$\know(d(now) = 1 \lor \ldots \lor d(now) = 10, \ins).$$ 
The observations obtained by the agent are described using a special function \( \obs. \) In the tree chop problem, we may have: \begin{equation}\label{eq exact sensing tree chop}
	\obs(a,s) = \begin{cases}
		{\it down} & a=getd\land d(s) = 0 \\ 
		{\it up} & a=getd\land d(s) \neq 0 \\
		{\it 0} & \textrm{otherwise} 
	\end{cases}
\end{equation}
which says that on doing the sensing action $getd$, if the tree is still standing, the sensor returns $up$, else it would return $down.$

A fixed successor state axiom for \( K \) then determines how the agent's knowledge changes over actions: 
\[\begin{array}{l}
	K(s', do(a,s)) \equiv \exists s''[K(s'', s) \land s' = do(a,s'')  \land Poss(a, s'')  \\ \hspace{2.5cm} ~\land (\obs(a,s'') =\obs(a,s)))]
\end{array}
\]
which has the effect of eliminating worlds that disagree with truth at the real world, leading to \emph{knowledge expansion.} %

\subsection*{Plans with loops} 
\label{sub:plans_with_loops}

We will be interested in computable program-like plans.  
We consider finite state controllers, which are fairly common in the literature \cite{DBLP:conf/aips/BonetPG09,DBLP:conf/aips/HuG13,DBLP:conf/ijcai/HuL11}.  


\begin{definition}\label{defn fsa} Suppose \( \A \) is a finite set of (parameterless) action terms, and \( \O \)  a finite set of objects, denoting {observations}. 
     A {finite memoryless plan} \( \X \) is a tuple \( \la \Q, \qi,\qf,  \gamma, \delta \ra \):\footnote{Assume terms such as \( chop \in \A \) and \( \set{up,down} \subseteq \O. \)}   \begin{itemize}
     \item \( \Q \) is a finite set of  {control states};
     \item \( \qi \in \Q \) is the  {initial}   state, and \( \qf\in \Q \) the final one;
          \item \( \gamma \in [\Q\su - \rightarrow \A] \) is a  {labelling} function for \( \Q\su - = \Q - \set{\qf} \);
          \item \( \delta \in [\Q\su - \times \O  \rightarrow \Q] \) is a {transition} function. 
     \end{itemize}

\end{definition} 

\begin{example} For Figure \ref{fig:tc}, we might have 
$\Q = \{\qi, Q, \qf\},$  $\gamma(\qi) = chop, \gamma(Q) = getd, \delta(Q, up) = \qi,$ $\delta(Q, down) = \qf,$ and $\delta(\qi, 1) = Q.$ 
\end{example}

Informally, we may think of applying these plans in an environment as follows: starting from \( \qi \) that advises the action \( \gamma(\qi) \), the environment executes that action and changes externally to return  \( o\in \O \).  Then, internally, we reach the control state \( \delta(\qi,o) \) and so on, until \( \qf \). To reason about these plans in an environment enabled in the situation calculus, we need to encode the plan structure as a \( \L \)-sentence, and axiomatise the execution semantics over situations, 
%
for which we follow \cite{DBLP:conf/ijcai/HuL11} and define: 

\begin{definition} Let \( \fsa \) be the union of the following axioms for: \begin{enumerate}
	\item domain closure for the control states: \(
		(\forall q) ~\{ q = \qi \lor q = Q_1 \lor \ldots \lor q = Q_n \lor q = \qf \}; 
	\) 
	\item unique names axiom for the control states: \(
		Q_i \neq Q_j \text{ ~for~ } i\neq j; 
	\)  
	\item action association:   \( \forall Q \in \Q\su - \) of the form \( \gamma(Q) = a \);  
	\item transitions:   \( \forall Q \in \Q\su - \) of the form \( \delta(Q,o) = Q' \).

\end{enumerate}
\end{definition}

\begin{definition}\label{defn execution semantics T}
		We use \( T\su *(q,s,q',s') \) as abbreviation for $\forall T [\ldots \supset T (q, s, q', s')]$, where the ellipsis is the conjunction of the  universal closure of: \begin{itemize}
		\item $T(q,s,q,s)$ 
		\item \( T(q,s,q'',s'') \land T(q'',s'',q',s') \supset T(q,s,q',s') \)  
		\item \(  \gamma(q)=a \land Poss(a,s) \land \obs(a,s)=o \) \\ \qquad \(\land~   \delta(q,o)=q' \supset T (q, s, q', do(a, s)) \).
	\end{itemize}
\end{definition}

In English: \( T\su * \) is  the reflexive transitive closure of the one-step transitions in the  plan. 

We are now prepared to reason about plan correctness: \begin{definition}\label{defn Levesque correctness} For any goal formula \( \phi\in\L, \) basic action theory \( \D, \) plan \( \X \) and its encoding \( \fsa, \) we say \( \X \) is correct for \( \phi \) iff \[
	\D\cup \fsa \models \forall s.~K(s,\ins) \supset \thereis s' [T\su * (\qi,s,\qf,s') \land \phi(s')].
\] 
\end{definition}

\begin{example} Let \( \D\sub {dyn} \) denote the tree chop axioms, and then it is easy to see that \( \D\sub {dyn} \cup \set{\eqref{eq:knowledge thickness discrete}, \eqref{eq exact sensing tree chop}} \cup \fsa \), where \( \fsa \) represents the encoding of Figure \ref{fig:tc}, is correct for the goal \( d = 0. \) 
\end{example}

\section{A Theory of Probabilistic Beliefs} 
\label{sec:stochastic_environments}

Our objective now is to generalise the above well-understood framework on knowledge and loopy plans to a stochastic setting. 
The account of knowledge, deterministic acting and exact sensing was extended by  
Bacchus, Halpern, and Levesque \cite{bacchus1999171}    --  BHL henceforth --   to deal with degrees of belief in formulas, and in particular, with how degrees of belief should evolve in the presence of noisy sensing and acting, in accordance with Bayesian conditioning. The main advantage of a logical account like BHL is that it allows a specification of belief that can be partial or incomplete, in keeping with whatever information is available about the application domain. The account is based on 3 distinguished fluents: \( p,l \) and \( alt \). The \( p \) fluent here is a numeric analogue to \( K \) in that \( p(s',s) \) denotes the weight (or density) accorded to \( s' \) when the agent is at \( s. \) Initial constraints about what is known  can be provided as usual. For example:  \begin{equation}\label{eq p ini grasp}
	p(\ivar, \ins) = \begin{cases}
		.1 & \textrm{if }(1 \leq d(\ivar) \leq 10) \\
		0 & \textrm{otherwise}
	\end{cases}
\end{equation} 
is saying that the tree's thickness $d$ takes a value that is uniformly drawn from  $\{1,\ldots,10\}$. We provide a definition for a belief modality $\bel$ shortly, but the above constraint can be equivalently written as: 
$$ \bel(d(now) = 1, \ins) = .1$$and so on for the other values. As argued earlier, the logical account also allows for uncertainty about the prior distribution -- as needed, for example, in \cite{DBLP:journals/ijrr/KaelblingL13} -- by means of expressions such as: 
$$ \bel(d(now) = 1, \ins) \geq .05$$ which says that any distribution where $d$ takes a value of $1$  with a probability greater than or equal to $.05$ is a permissible one. 

The \( l \) fluent is used to express the likelihoods of outcomes. For example, suppose we had a sensor that would inform the agent about the numeric value of the $d$ fluent in a situation. Then an axiom of the form: 
\begin{equation}\label{eq sonar gaussian}
	l(getd(z),s) = \N(z;d(s),.25)
\end{equation}
says that the 
observed value on the sensor is normally distributed around the  true value, with a variance of .25. In robotics terminology \cite{thrun2005probabilistic}, the sensor is said to have a Gaussian error profile. 

To handle noisy actions, the idea is that if \( chop(x) \) represents a chop that decrement's the tree's thickness by $x$ units, assume a new action type \( chop(x, y) \) in that $x$ is the intended argument and  \( y \) is taken to be what actually happens. These are chosen by nature, and as such, out of the agent's control. These action types are retrofitted in successor state axioms: \begin{equation} \begin{array}{l}
	d(do(a,s)) =u \equiv (a=chop(x,y) \land u = d(s) -y) ~\lor \\ \quad\quad\quad\quad\quad\quad\quad\quad\quad\quad (a\neq chop(x,y) \land u = d(s))
\end{array}	
\end{equation}
says that \( d \) is actually affected by the second argument. 

Since the agent is assumed to not control \( y, \) we use \( alt \) to model the possible \emph{alternatives} to an intended action:\footnote{A more involved version would introduce {\it alt} as a fluent, allowing possible outcomes to be determined by a situation.} 
%
\begin{equation}\label{eq alt mv}
	alt(chop(x,y),a',z) \equiv a' = chop(x,z)
\end{equation}
says that, for example, $chop(1,2)$ is indistinguishable from $chop(1,3)$. To suggest that some alternatives are more likely than others, an axiom like 
%
\begin{equation}\label{eq likelihood noisy move}
	l(chop(x,y),s) = \N(y;x,.25)
\end{equation}
then says that the actual value is normally distributed around the intended value, with a variance of .25.  In robotics terminology, this is the equivalent to an action having additive Gaussian noise.  

Likelihoods and \( \alt \)-axioms  determine the probability of successors, enabled by the following successor state axiom: \begin{equation*}\label{eq:p ssa}
	\begin{array}{l}
		p(s',do(a,s)) = u \,\,\equiv\\
		\quad \quad \exists a', z, s''~[\alt(a,a',z) \land s' = do(a',s'') \land \poss(a',s'')    \,\land\\
        \quad\quad\quad \quad\quad\quad\quad\quad u = p(s'',s) \times l(a',s'')] \\
	 \lor\,\, \lnot\exists a', z,s''\\ \quad \quad ~[\alt(a,a',z) \land s' = do(a',s'') \land \poss(a',s'')
           \land u = 0]
	\end{array}
\end{equation*}
%
%
%
%
which essentially says that if two situations $s$ and $s'$ are considered epistemically possible, on doing $a$ at $s$, $do(a,s)$ and  $do(b, s')$  will also be considered epistemically possible, where $b$ is any $alt$-related action to $a$. Moreover, the $p$-value of $do(b,s')$ is that of $s'$ multiplied by the likelihood of the outcome $b$. 

Putting all this together, the degree of belief in \( \phi \) at \( s \) is defined as the weight of worlds where \( \phi \) is true: \[
\displaystyle	 \bel(\phi,s) \doteq \sum\sub {\set{s'\colon \phi(s')}} p(s',s) ~~\bigg/~~ \sum\sub {s'} p(s',s)
\]
We write \( K(s',s) \) to mean \( p(s',s)\gt 0, \) and \( \know(\phi,s) \doteq \bel(\phi,s) = 1. \) Finally, note that when the likelihood models are trivial, that is: $\forall a, s.~l(a,s) =1$, we are in the setting of nondeterministic but non-probabilistic   acting and sensing. 


\section{Controllers With Noisy Acting} 
\label{sub:plans_with_loops_revisited}

Our first objective will be to motivate a definition for analysing the correctness of plan structures when actions are noisy (that is, they are non-deterministic, and the actual outcome is not directly observable). We assume, for now, that sensing is exact. This then also  handles the case where actions are non-deterministic but observable immediately after, by way of a sensing action to inform the planner about the outcome.  
As far as the syntax of the plan structure goes, we will not want it to be any different than Definition \ref{defn fsa}; however, we will need to revisit Definition \ref{defn execution semantics T} to internalise the noisy aspects of acting by using $alt$.

\begin{definition}\label{defn execution semantics U}
		We use \( U\su *(q,s,q',s') \) as abbreviation for $\forall U [\ldots \supset U (q, s, q', s')]$, where the ellipsis is the conjunction of the  universal closure of: \begin{itemize}
		\item $U(q,s,q,s)$ 
		\item \( U(q,s,q'',s'') \land U(q'',s'',q',s') \supset U(q,s,q',s') \)  
		\item \(  \gamma(q)=a \land \thereis b,z~ (\alt(a,b,z) \land  Poss(b,s) \land \obs(b,s)=o\land \delta(q,o)=q') \supset U (q, s, q', do(b, s)) \).
	\end{itemize}
\end{definition}
The main new ingredient here over \( T\su * \), of course, is how control states transition from a situation to a successor. The reflexive transitive closure of \( U \) basically says that if the controller advises \( a \) and \( b \) is any action that is \( \alt \)-related to \( a, \) we consider the transition wrt the executability and the sensing outcome of the action \( b. \) The idea, then, is allow \( U\su * \) to capture the least set that accounts for all the successors of a situation where a noisy action is performed. We define:


\newcommand{\us}{{U\su *}}

\begin{definition}\label{defn theta  correctness} For any goal formula \( \phi\in\L, \) basic action theory \( \D, \) plan \( \X \) and its encoding \( \fsa, \) we say \( \X \) is correct for \( \phi \) iff \[
	\D\cup \fsa \models \forall s.~K(s,\ins) \supset \exists s'[U\su *(\qi,s, \qf, s') \land \phi(s')]. 
\] 

\end{definition}
It can be shown that the new semantics coincides with Levesque's account when the action theory is \emph{noise-free}: that is,  \( \alt \)-axioms are trivial \( \forall z (\alt(a,a',z) \equiv a = a') \), and \( l \) mimics the behavior of \( \obs \) in assigning 1 to situations that agree with the sensing outcome and 0 to those that disagree.
 Then: 
\begin{theorem} Suppose \( \D \) is a noise-free action theory, 	\( \X \) and \( \fsa \) are as above, and \( \phi \) is any situation-suppressed formula not mentioning the fluent \( K. \) Then, \( \X \) is correct for \( \phi \) in the sense of Definition \ref{defn Levesque correctness} iff \( \X \) is correct in the sense of Definition \ref{defn theta  correctness}.
\end{theorem}

\begin{proof} \( \us \) differs from \( \ts \) is only one aspect, that of \( \alt \)-related actions governing the transition to a successor situation. By assumption, \( \forall z(\alt(a,a',z) \equiv a=a') \);  so Definition \ref{defn Levesque correctness}'s constraint on \( \ts \) coincides with  Definition \ref{defn theta  correctness}'s constraint on \( \us \). \qed

\end{proof}

Definition \ref{defn theta  correctness} only tests for a single goal-satisfying path, which is often referred to as a \emph{weak plan} \cite{Cimatti200335}. 
A property like {\it termination} could be formalised using: \begin{equation}\label{eq ter adequate}
	\begin{array}{l}
		\D\cup \fsa \models \forall s.~K(s,\ins) ~\supset \\ \hspace{.5cm}	\forall s'~[\us(\qi, s, 	q, s') \supset \thereis s''~(\us(q,s',\qf, s''))]. 
	\end{array}
\end{equation} 

It is well-known \cite{Cimatti200335} that in the absence of nondeterminism, termination is implied by the existence of a goal-satisfying path:\footnote{There are, of course, a number of other criteria in terms of which one characterises plan execution \cite{Cimatti200335}, a discussion of which is orthogonal to the issues of interest here and are hence omitted. Major criteria include \textit{fairness}, where if a nondeterministic action is executed infinitely many times then every  outcome is assumed to occur infinitely often, and \textit{acyclicity}, where the same state is not allowed to be visited twice in plan execution paths.}
 

\begin{proposition} Suppose \( \D, \X \)  and \( \phi \) are as above. If \( \X \) is correct for \( \phi \) in the sense of Definition \ref{defn Levesque correctness} then \eqref{eq ter adequate} holds.
	
\end{proposition}

\begin{example}\label{ex noisy acting exact sensing tc} Imagine a tree chop problem with noise-free sensing (i.e., let \( \obs \) work as in \eqref{eq exact sensing tree chop}), but with noisy actions. Let \( chop \in \A \) correspond to the \( \L \)-action \( chop(1,1) \), with  \begin{equation}\label{eq likelihood noisy chop}
	l(chop(x,y),s) = \begin{cases}
		.9 & \textrm{if } x=y \\ 
		.1 & \textrm{if }y=0 \\ 
		0 & \textrm{otherwise}
	\end{cases}
\end{equation}
That is, the chop action does nothing with a small probability. Letting the initial theory be a \( p \)-based axiom for \eqref{eq:knowledge thickness discrete}, we see that the plan from Figure \ref{fig:tc} is correct in the sense of Definition \ref{defn theta  correctness}, and it is also a terminating plan, because regardless of how many times the action fails, there is clearly one execution path, that of the appropriate number of chops always succeeding, which stops after enabling the goal. 

	
\end{example}



While Definition \ref{defn theta  correctness} looks at every non-zero initial world, weaker specifications are possibile still. For example:\begin{equation}
	\D \cup \fsa \models \forall s.~p(s,\ins) \gt \kappa \supset \exists s'[U\su *(\qi,s, \qf, s') \land \phi(s')] 
	\tag{\ddag}
\end{equation}
looks at worlds with weights  \( \gt \kappa \) whereas 
\begin{equation}
	\D \cup \fsa \models \bel(\exists s'[U\su *(\qi,now, \qf, s') \land \phi(s')], \ins) \geq \kappa 
	\tag{$\sharp$}
\end{equation}
says that the sum (or integral) of initial worlds where there is a weak plan is \( \geq \kappa. \) Of course, since $K(s', s)$ is an abbreviation for $p(s',s) \gt 0$, and $\know(\phi,s) \doteq \bel(\phi,s) = 1$, we have: 

\begin{proposition} Definition \ref{defn theta  correctness} is equivalent to ($\ddag$) for $\kappa = 0$, and it is equivalent to ($\sharp$) for $\kappa = 1$. 

\end{proposition}



\begin{example} Imagine a tree chop problem with two types of trees, wooden ones and metal ones; the chop action has no effect on a metal tree \cite{DBLP:conf/kr/SardinaGLL06}. Suppose we have 3 worlds: the first with a wooden tree of thickness 1 and weight .4, the second with a wooden tree of thickness 2 and weight .4, and the third with a metal tree of arbitrary non-zero thickness with weight .2. If the goal is  \( d=0, \) and the plan is the one from Figure \ref{fig:tc}, then \( (\ddag) \) holds for (say) \( \kappa= .3  \) and \( (\sharp) \) holds for (say) \( \kappa = .7. \)

	%
	%
	
\end{example}

We do not think that there is one preferred definition for correctness. While Definition \ref{defn theta  correctness} attempts correctness in the sense of \cite{Cimatti200335}, planning for likely states, as in $(\ddag)$,  is very common in robotics when exploring large biased search spaces \cite{doi:10.1177/0278364904045471,Knepper2017}.

\section{Incorporating Noisy Sensors} 
\label{sec:belief_based_planning}

Accounts like Definition \ref{defn theta  correctness} and ($\sharp$) capture nondeterministic acting when the outcome of the action is immediately observable. In applications such as robotics, sensing is often noisy. 
 Similar to contingent planning \cite{DBLP:conf/aips/BonetG00,petrick2004extending,muise-aaai-15}, we will now motivate a  semantics of plan execution over belief states. 

To review the setting informally, consider a noise-free tree chop problem instantiated by \eqref{eq:knowledge thickness discrete}. Although the agent believes that \( d \in \set{1,\ldots,10} \), in the real world \( \ins, \) the tree has a fixed thickness, say 2. In this case, the controller discussed previously will advise two chop actions, which will be executed in each of the possible worlds, including \( \ins. \) 
At this point, a noise-free sensor returns \( down, \) and so, the agent will come to believe that the tree is down. Implicitly, all the worlds other than \( \ins \) considered possible initially will be discarded, as they are no longer compatible with the sensing results. (For example, the world in which  $d=1$ will be discarded after  the sensor says that the tree is still standing on doing a chop action, because that is clearly not possible in such a  world.) 
 However, a noisy sensor might return a $down$ despite the tree still standing. The point, then, is that the sensor's error profile will inform the agent how likely it is that a $down$ is observed when the tree still stands, and based on that, subsequent actions can be taken until it is believed that the tree is no longer standing.

To formalise this intuition, we will need to address two technical issues. 
First, observe that the account of \( \bel \) from BHL \cite{bacchus1999171} makes no mention of sensing functions, and in this sense, the language is only geared for projection. That is, we can infer  the value of \( \bel(d\leq 4, do(getd(3),\ins)) \) where we explicitly provide the sensor reading, but 
%
it is ill-formed in the language to reason about beliefs after \( getd \) sans argument that is determined only at run time. Moreover, as discussed above, the external feedback (\eg number reported on a sensor) will only be an estimate of the true property when we turn to   noisy sensors.

The solution to this issue by BHL was to define programs for sensing actions: for example, a sensor was defined as $\pi y.~getd(y)$, the latter being a GOLOG program \cite{reiter2001knowledge} that non-deterministically chooses the argument for the sensing action. In \cite{Belle:2015ab}, 
a slightly simpler technical device was introduced where $getd$ also stood for a program, but the semantics of program execution, which is defined over action sequences, incorporates run-time sensor readings. 
Neither of these solutions is appropriate for us, because: (a) we would like to avoid the complexity of defining GOLOG programs; and (b) we would like the new definition to be compatible with our previous accounts of plan execution, enabled via the $\obs$ function.  We achieve this by considering a  \emph{runtime sensing outcome function} \( \Pi\colon \A^ * \rightarrow  \O\). Recall that in the situation calculus, situations are not states, and the initial situation corresponds to the setting where the agent has not executed any action. In a way, $\Pi$ is an analogue to usual definitions of observation functions that map states to observations in that it responds to an action history. Then, we use $\obs$ to refer to runtime readings as follows:  \[ \begin{array}{l}
	\obs(a,do(a_k, do(\ldots, do(a\sub 1,\ivar)\ldots))) = \Pi(a\sub 1, \ldots, a\sub k,a).	
\end{array}
\]
With this machinery, we can use $\obs$ as usual in the plan execution semantics.\footnote{A more involved characterisation for $\Pi$ would also take into account the values of fluents at situations, which we omit here for simplicity.} We can then introduce parameterless sensing actions like \( getd \) whose likelihood is now defined to mimic \eqref{eq sonar gaussian} as follows: \begin{equation}
	l(getd,s) =   \N(\obs(getd,s);d(s),1). 
\end{equation}

The second technical issue is to interpret execution paths over belief states, but while referencing the real world to test for sensing outcomes. That is, starting from a control state (from the plan structure) and the agent's beliefs, we will need to define how a new control state is reached with an updated set of beliefs. So, we will need to ``reify'' beliefs in formulas: for any ground situation term \( s \), we introduce a new term \( \overline s \) to 
be used with formulas in that we write \( \phi(\overline s) \) to mean \( \forall s'.~K(s',s) \supset \phi(s'). \)  We will often use two such terms in a predicate for one-step transitions:  $V(t, \overline s, t', \overline s')$  can be first expanded to $\forall s^*.~K(s^*, s) \supset V(t, s^*, t', \overline s')$, which then expands to $\forall s^*, s^+.~[K(s^*, s) \land K(s^+, s')] \supset V(t, s^*, t', s^+)$. Recall that the $K$ fluent was assumed to be an equivalence relation, and so, roughly speaking, $V(t, \overline s, t', \overline {do(a,s)})$ can be seen as saying that starting from $t$ and the belief state given by the situation $s$ (that is, all  $K$-related situations from $s$), we perform a transition to $t'$ and the belief state given by $do(a,s)$. Formally, 
we define: 

\begin{definition}\label{defn execution semantics V}
		We use \( V\su *(q,\overline s,q',\overline {s'}) \) as abbreviation for $\forall V [\ldots \supset V (q, \overline s, q', \overline {s'})]$, where the ellipsis is the conjunction of the  universal closure of: \begin{itemize}
		\item $V( q,\overline s,q, \overline s)$ 
		\item \( V(q,\overline s,q'', \overline {s''}) \land V(q'',\overline {s''},q',\overline {s'}) \supset V(q, \overline s,q', \overline {s'}) \)  
		\item \(  \gamma(q)=a \land Poss(a, \overline s)  \land \obs(a,s) = o\) \\ \qquad \(  \land~ \delta(q,o)=q' \supset V ( q, \overline s, q', \overline {do(a, s)}) \).
	\end{itemize}
    
\end{definition}
The one-step transition  is based on the controller advising \( a, \)   this action being executable at all   accessible worlds, and the sensing function returning \( o \) for \( a \) at   \( s \), which is taken to be the real world. Most significantly, observe that \( \phi(\overline {do(a,s)}) \), by means of \( p \)'s successor state axiom, would implicitly account for all the \( \alt \)-related actions to \( a. \)\footnote{This is a key point as far as practical planning frameworks are concerned: $V^*$ does not look so different from the semantics of noise-free belief-based planning -- see the account in \cite{DBLP:conf/kr/SardinaGLL06}, for example; however, what is then needed is a set of belief states that correctly accounts for the unobservability of nondeterministic outcomes and how that changes with noisy sensing.} 

With this, we are prepared to reason about correctness: 

\begin{definition}\label{defn belief based correctness} For any goal formula \( \phi\in\L, \)  \( \D, \)  \( \X \) and its encoding \( \fsa, \) we say \( \X \) is epistemically correct for \( \phi \)  iff \[
	\D\cup \fsa \models \forall s.~K(s,\ins) \supset \exists s'~[V\su *(\qi, \overline {s}, \qf, \overline {s'}) \land \phi(\overline {s'})]
\] 
\end{definition}

%
%

\begin{figure}[t]
  \centering
    \includegraphics[width=.35\textwidth]{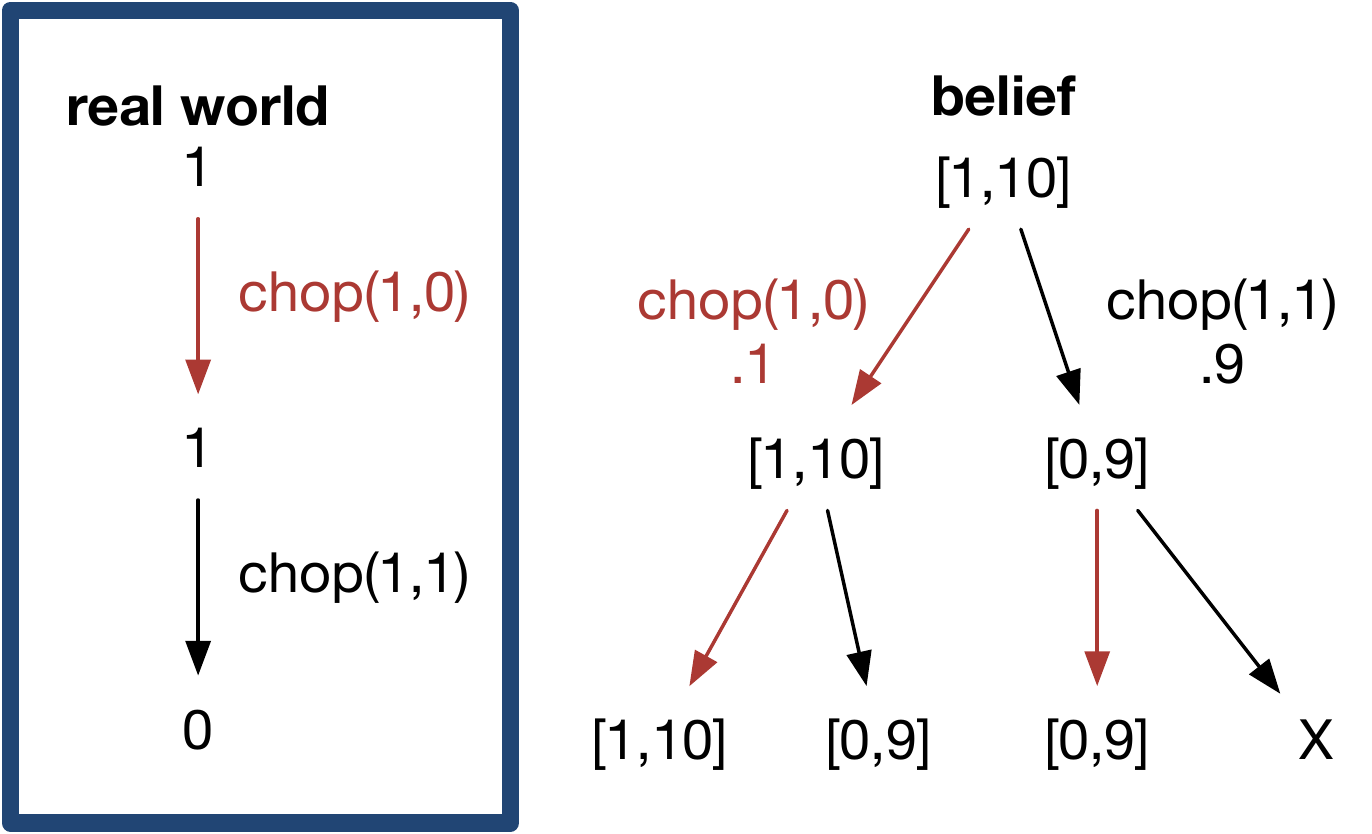}
  \caption{execution path for the tree chop problem with exact sensors against the controller from Figure \ref{fig:tc}}
  \label{fig:grasp-exact-sensor}
\end{figure}

Before turning to the case of noisy sensors, let us revisit the tree chop problem with noisy acting and exact sensing to better understand how belief state transitions work.  

\begin{example}\label{ex belief based exact sensing} Let  \( \D \) be an action theory built from \eqref{eq p ini grasp}, \eqref{eq alt mv}, and the likelihood axiom \eqref{eq likelihood noisy chop}. Suppose the sensor works as follows: \begin{equation}
	l(getd,s) = \begin{cases}
		1 & \obs(getd,s) = down \land d=0 \\
		1 & \obs(getd,s) = up  \land d\gt 0 \\ 
		0 & \textrm{otherwise}
	\end{cases}
\end{equation}
It is noise-free. Finally, suppose our goal is \( \bel(d<10,now) \gt .9 \).

A plan that is epistemically correct for this goal is given in Figure \ref{fig:grasp-exact-sensor}.  We only argue for the initial world \( \ins. \) 
It also depicts a possible execution path of the plan using \( V\su * \).  Let us suppose  \( d(\ins)=1. \)  The first action advised by the controller is \( chop, \) and suppose the action instantiates as \( chop(1,0) \). Then \( d(do(chop(1,0),\ins)) = 1; \) so its still standing, and also, the agent accords a belief of .9 to \( [0,9] \) that corresponds to a successful move, and a belief of .1 to its failure. The noise-free sensor naturally returns {\it up.} The controller advises \( chop \) again, and suppose this time, the action instantiates as \( chop(1,1) \). 

Incidentally, even without a sensor reading, the degree of belief in $d\lt 10$ is $\gt .9$ because the only branch that still entertains $d$ retaining a value of 10 is that of both chop actions failing, with a likelihood of $.1\times .1$. In any case, the controller advises a sensing action, which would return {\it down}, and so the plan terminates for \( \alpha =[chop(1,0)\cdot getd\cdot chop(1,1)\cdot getd] \) with \( \know(\bel(d\lt 10, now) \gt .9, do(\alpha,\ins)). \)
	
\end{example}


	\begin{figure}[t]
	  \centering
	    \includegraphics[width=.45\textwidth]{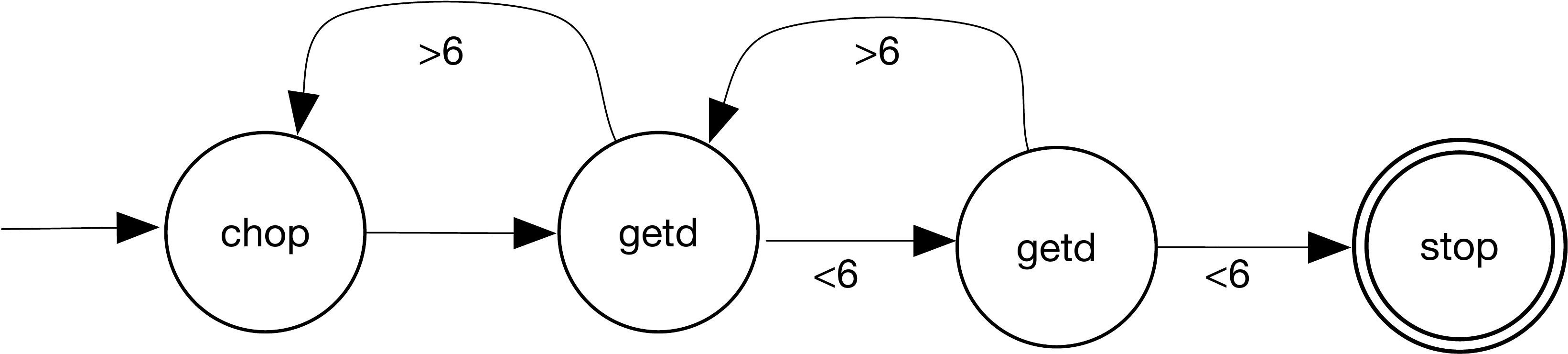}
	  \caption{a controller for chopping the tree with noisy sensing}
	  \label{fig:grasp-plan}
	\end{figure}

\begin{example}\label{ex grasping with noisy sensing} Consider the tree chop problem with noisy sensors and effectors. Suppose the initial theory and \( \alt \)-axioms are \eqref{eq p ini grasp} and \eqref{eq alt mv} as before, but the likelihoods for the effector is given by \eqref{eq likelihood noisy move} and that for the sensor is given by \eqref{eq sonar gaussian}. Finally, suppose our goal is \( \bel(d\leq 5,now) \gt .8 \). 
	
	A  plan that is epistemically correct for the goal is given in Figure \ref{fig:grasp-plan} wrt observed values of 5.5, 4.5 and 3.9. We only argue for \( \ins \). Assume also that $\{ \lt 6, \gt 6\} \in \O$ and that the numeric values obtained from the sensor map to these binary outcomes.
    
	Here, the controller advises \( chop \), after which the belief in \( \psi= d\leq 5 \) is \( \gt .5. \) This is because the likelihood of the action succeeding is more than it failing, and given the prior in $d\leq 5$ is .5, the posterior should be clearly greater than $.5$. 
Suppose now the sensed value is 5.5. The belief in \( \psi \) drops to \( \lt .5 \). The controller advises another \( chop \), at which point the belief in \( \psi \) increases to \( \gt .5. \) Next, we observe a   reading of 4.5 followed by 3.9. The controller terminates. On termination, it can be verified that 
the robot knows that the degree of belief in \( \psi \) is \( \gt .8 \).

\end{example}

\renewcommand{\vs}{{V\su *}}


In a noise-free setting, our definition of an  epistemically correct plan is downward compatible with Definition \ref{defn Levesque correctness} (and thus, Definition \ref{defn theta  correctness}): 

%
%

\begin{theorem}\label{thm belief based compatibility noise free} Suppose \( \D \) is a noise-free action theory, \( \X,\fsa \) as above, and \( \phi \) is any formula not mentioning \( K. \) If \( \X \) is epistemically correct for \( \phi, \) then it is correct for \( \phi \) in the sense of Definition \ref{defn Levesque correctness}.

\end{theorem}

\begin{proof} Suppose \( \X \) is epistemically correct but not correct. Then there is some \( s \) such that \( K(s,\ins) \) and \( \neg\thereis s'' \ts(\qi,s,\qf,s'') \land \phi(s''). \) By assumption, \( \vs(\qi, \overline s, \qf, \overline {s'}) \land \phi(\overline {s'}) \) for some \( s' \). Thinking of \( \langle \)control states, situations\( \rangle \) are ``nodes" in an execution path, the definition of \( \vs \) is the least set of pairs of nodes containing: \( \langle \qi,t \ra \) for all situation terms \( t \) such that \( K(t,s) \), which includes \( s \) because \( K \) is assumed to be an equivalence relation; \( \la q,do(a\sub 1,t)\ra \) for all situation terms \( t \) such that \( K(t,s) \) provided \( \qi \) advises \( a\sub 1 \), it is executable at every \( t \) and the sensing function returns \( o \) for \( a\sub 1 \) at \( s \) and \( \delta(\qi,o) = q \); and so on.
	 By assumption, \( \vs \)  contains  as node \( \la\qf,s'\ra \) for \( s' = do(a\sub 1\cdots a\sub k,s) \). But, by the definition of \( \ts, \) it  follows that \( \ts(\qi,s,\qf,s'). \) Moreover, since \( \phi(\overline {s'}) \) and \( K \) is reflexive, \( \phi(s') \). Contradiction. \qed

\end{proof}

\begin{figure}[t]
  \centering
    \includegraphics[width=.4\textwidth]{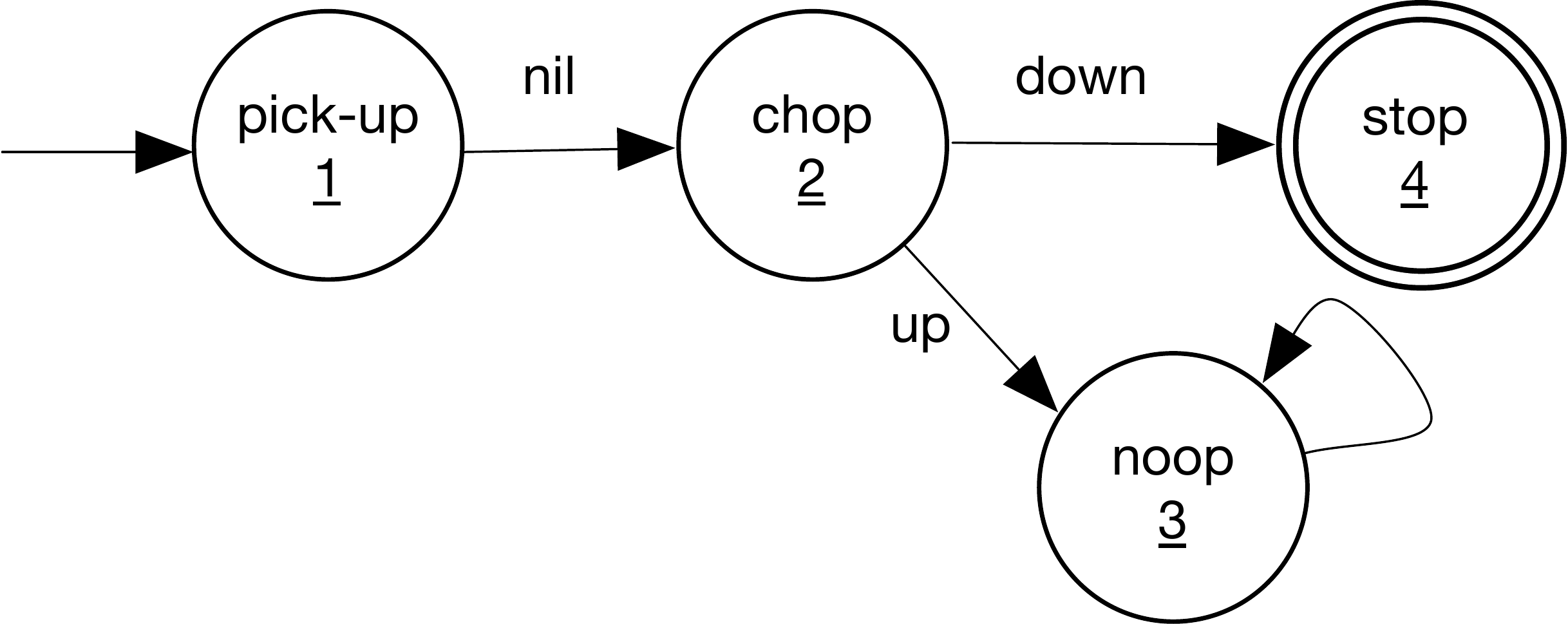}
  \caption{a problematic controller}
  \label{fig:beliefnotpc}
\end{figure}


In general, in the presence of noise, as one would expect, epistemic correctness diverges from correctness criteria based on plan evaluation at initial worlds. For example, we have: 



\begin{theorem}\label{thm epistemic does not entail ter} Suppose \( \D \) is any action theory,  \( \X,\fsa \) as above, and \( \phi \) is any formula not mentioning \( K. \) 
	If \( \X \) is epistemically correct for \( \phi, \) then it does not follow that \( \X \) satisfies \eqref{eq ter adequate}. 
	
\end{theorem}

\begin{proof} To prove this result, it suffices to provide a (possibly unwise) controller that is epistemically correct, but does not satisfy \eqref{eq ter adequate}.  Consider the tree chop problem for a tree of unit thickness, but with an additional action for picking up a saw. Suppose there is only a single initial world \( \ins, \)  and \( \forall a, s(Poss(a,s)\equiv \true). \) Suppose the pick-up can fail, that is, suppose $alt(pickup, b, z) \equiv b = noop$, where $noop$ is the action of doing nothing. Suppose neither of these actions provide the agent with any meaningful sensing result: that is, $\forall s~ \obs(pickup, s) = \obs(noop, s) = nil$. 
Imagine a controller like in Figure \ref{fig:beliefnotpc}, where the control states are designated  by numbers. (That is, $q_1$ is the control state advising $pickup$ and $q_4$ is terminating state.)  
Initially, the controller advises $pickup$, leading to two successor situations: \( do(pickup,\ins) \) and \( do(noop,\ins) \). By way of the definition of $\obs$ and $Poss$ for these actions, we have \( \vs(q\sub 1, \overline {\ins}, q\sub 2, \overline {do(pickup,\ins)}) \), \( \us(q\sub 1, \ins, q\sub 2, do(pickup,\ins)) \) and \( \us(q\sub 1, \ins, q\sub 2, do(noop,\ins)) \). Suppose now \( \obs(chop, do(pickup,\ins)) = down \) and so we have \\ \( \vs(q\sub 1, \overline {\ins}, q\sub 4, \overline {do(pickup \cdot chop, \ins)}) \), and by construction, \( [d(now) = 0](\overline{do(pickup\cdot chop, \ins)}). \)  However, suppose \( \obs(chop, do(noop,\ins))  = up.\) Then we have \( \us(q\sub 1, \ins, q\sub 3, do(noop\cdot chop, \ins)) \), which will not terminate. \qed

\end{proof}



The intuitive reason is that the termination  conditions defined over \( \us \) respond to sensing results along every execution path, whereas \( \vs \) only responds to the sensing outcomes for the path of advised actions from an initial world and how belief changes with it. This can be seen to be not surprising: among other things, the much stronger \eqref{eq ter adequate} is not needed because incompatible worlds will get discarded after sensing. 


Let us conclude the section with two remarks. First,  as mentioned earlier, the definition of $V^*$ does not look very different from the semantics of noise-free belief-based planning, and this is good news: if the space of belief states is designed carefully to account for noise, algorithms for noise-free generalised planning may carry over to the stochastic case. Second, note that $V\su *$ was defined to include an explicit reference to sensing outcomes from the environment, which is external to the agent. 
A result in \cite{DBLP:conf/kr/SardinaGLL06} shows that it is not possible to realise that we are making progress towards the goal without such a construction in belief-based planning.

\section{Discussion and Conclusions} 
\label{sec:related_work}

Generalising plans has been of interest since the early days of planning \cite{fikes1972learning}. Algorithmic proposals to synthesise plans that generalise varied widely in methodology, ranging from interactive theorem proving \cite{DBLP:conf/aips/StephanB96} to learning from examples \cite{winner07_loop}. The convergence of these approaches to synthesise plans that solve multiple problem instances is a recent effort  \cite{DBLP:conf/aaai/BonetPG10,siddthesis,levesque2005planning,DBLP:conf/ijcai/HuL11,DBLP:conf/aips/HuG13}. We refer interested readers to \cite{siddthesis} for a comprehensive list of references, and 
\cite{srivastava2015tractability,DBLP:conf/ijcai/BonetGGR17} for recent advances on handling nondeterminism. Outside of Levesque's account on the correctness of program-like plans as an epistemic formulation, there are numerous variants \cite{DBLP:journals/ai/LinD98,Cimatti200335,Belle:2016ab,srivastava2015tractability,DBLP:conf/ijcai/BonetGGR17}. The semantics $U^*$ extended Levesque's $T^*$ to handle  noisy acting, and $V^*$ further extends that to noisy sensing, thereby obtaining a full generalisation of the formulation to handle nondeterminism.

Belief-based planning, which we touch upon, is widely studied,  e.g., \cite{DBLP:conf/aips/BonetG00}, and the usual approach is to formulate a nondeterministic (conformant) planning problem that treats belief states as first-class citizens. Our definition of \( \vs \) can be seen as a formalisation of this semantics against a logic of probabilistic belief and action. In that regard, the motivation behind this work is close in spirit to knowledge-based programs \cite{DBLP:journals/tocl/Reiter01,DBLP:journals/sLogica/LesperanceLLS00} and its stochastic extension \cite{Belle:2015ab}. These are formulated using the situation calculus and the high-level programming language GOLOG, but, of course, variant languages are also popular for developing such planning accounts \cite{thielscher:ICAART10}. At the outset, 
 there are significant reasons to develop a semantics customised to memoryless plans, as we argue below. Moreover, since there are a number of generalised planning algorithms that synthesise loopy plans \cite{DBLP:conf/aips/HuG13}, an execution semantics tailored to that representation is useful to understand how those algorithms can be applied to domains with noise. 
 
 Let us begin by observing that memoryless plans can be easily encoded as GOLOG programs consisting of  atomic physical actions and branches based on sensing outcomes. In general, given a program $\delta$, one is interested in showing that \[\D \cup \Theta \models Do(\delta, \ins, do(\sigma, \ins)) \] where $\Theta$ encodes the single-step transition semantics of $\delta$, and $\sigma$ is a ground sequence of actions such that $\delta$ terminates in $do(\sigma,\ins)$. The key feature of knowledge-based programs is that $\delta$ can mention $\know$, and the probabilistic belief operator  $\bel$ in \cite{Belle:2015ab}. Nonetheless, note that if $\delta$ does not mention $\bel$, the entailment criteria above is weaker than Definition \ref{defn Levesque correctness} as it only looks at $\ins$. But if it mentions the $\bel$ operator, then it seems closer to Definition \ref{defn belief based correctness}, but at the cost of a more cumbersome plan structure: $\delta$ can have unbounded memory (via while loops), can refer to complicated state properties, and is subjective, whereas Definition \ref{defn belief based correctness} is defined for  memoryless plans built purely from a finite set of atomic actions.

So, in the current paper, the end result is an account of correctness with widely-studied loopy plan structures, which eschews the complications of GOLOG but is able to achieve almost as much.
Naturally, then, relating knowledge-based programs and loopy plans (in belief-based settings) is likely to be of considerable theoretical interest, as would a closer study of the two execution semantics. 
 (Cf. also   \cite{DBLP:journals/ai/LinD98} on memoryless structures being effective, and \cite{DBLP:journals/sLogica/LesperanceLLS00} on knowing how to execute GOLOG programs.) 


 Despite focusing on probabilities and nondeterminism, this paper has established no connection  to the large body of work on Markov decision processes \cite{puterman}. 
 Mostly,  decision-theoretic planning frameworks are characterised in terms of optimality criteria against expected rewards, often enabled via dynamic programming, while we have treated goals as arbitrary formulas that are to be satisfied, as would symbolic planning frameworks such as \cite{Cimatti200335}. Nonetheless, one can imagine ways of recasting expected rewards in terms of goal satisfaction or vice versa \cite{series/synthesis/2013Geffner}, and that is arguably worth doing in the context of this paper so as to relate to efforts such as   \cite{poupart2004bounded}. Interestingly, recent robotics planners such as \cite{DBLP:journals/ijrr/KaelblingL13} eschews a planning paradigm that advises actions for every belief state, as one would in partially observable Markov decision processes, and instead resorts to a scheme that computes plans for a designated initial belief state, as in belief-based planning. Moreover, as mentioned before, ultimately the goal here was to generalise Levesque's account and to provide a rigorous foundation for extensions such as the handling of non-unique prior distributions.

As a final remark, like in Levesque's original formulation, one can motivate a generic planning procedure as follows: \begin{itemize}
	\item[] \textbf{input:}  \( \phi, \) \( E\su * \in \set{\ts,\us,\vs}, \Delta \) is a correctness criteria 
	\item[] \textbf{repeat with} \( \X \in \) {FINITE STATE CONTROLLERS} 
	\item[] \quad \textbf{if} \( \D\cup \fsa \models \forall s.~K(s,\ins) \supset \Delta(E\su *, \phi, s) \) \textbf{then} \textbf{return} \( \X \)
\end{itemize}
Naturally, we do not expect to use a full-blown  logical framework for planning, nor do we expect planners to actually use such a procedure in practise. Languages like the ones in \cite{Sanner:2010fk,DBLP:journals/ijrr/KaelblingL13} seem entirely reasonable. It is also conceivable that existing algorithms for generalised planning, such as bounded AND/OR searches, can be adapted for stochastic settings, as argued earlier, possibly by leveraging   abstraction techniques  
\cite{srivastava2015tractability,DBLP:conf/ijcai/BonetGGR17}. We hope this paper is also useful for approaching and resolving that line of  inquiry. 




\bibliographystyle{unsrt}

\begin{thebibliography}{10}

\bibitem{DBLP:conf/aaai/Levesque96}
H.~J. Levesque.
\newblock What is planning in the presence of sensing?
\newblock In {\em Proc. AAAI / IAAI}, pages 1139--1146, 1996.

\bibitem{Cimatti200335}
A.~Cimatti, M.~Pistore, M.~Roveri, and P.~Traverso.
\newblock Weak, strong, and strong cyclic planning via symbolic model checking.
\newblock {\em Artificial Intelligence}, 147(1--2):35 -- 84, 2003.

\bibitem{DBLP:conf/aips/BonetPG09}
B.~Bonet, H.~Palacios, and H.~Geffner.
\newblock Automatic derivation of memoryless policies and finite-state
  controllers using classical planners.
\newblock In {\em ICAPS}, 2009.

\bibitem{DBLP:conf/aips/HuG13}
Y.~Hu and G.~{De Giacomo}.
\newblock A generic technique for synthesizing bounded finite-state
  controllers.
\newblock In {\em ICAPS}, 2013.

\bibitem{srivastava2015tractability}
S.~Srivastava, S.~Zilberstein, A.~Gupta, P.~Abbeel, and S.~Russell.
\newblock Tractability of planning with loops.
\newblock In {\em AAAI}, 2015.

\bibitem{DBLP:conf/aips/KominisG15}
F.~Kominis and H.~Geffner.
\newblock Beliefs in multiagent planning: From one agent to many.
\newblock In {\em ICAPS}, pages 147--155, 2015.

\bibitem{muise-aaai-15}
C.~Muise, V.~Belle, P.~Felli, S.~McIlraith, T.~Miller, A.~Pearce, and
  L.~Sonenberg.
\newblock Planning over multi-agent epistemic states: A classical planning
  approach.
\newblock In {\em Proc. AAAI}, 2015.

\bibitem{mataric2007robotics}
M.~J. Matari{\'c}.
\newblock {\em The robotics primer}.
\newblock Mit Press, 2007.

\bibitem{citeulike:528170}
R.~B. Scherl and H.~J. Levesque.
\newblock Knowledge, action, and the frame problem.
\newblock {\em Artificial Intelligence}, 144(1-2):1--39, 2003.

\bibitem{reasoning:about:knowledge}
R.~Fagin, J.~Y. Halpern, Y.~Moses, and M.~Y. Vardi.
\newblock {\em Reasoning About Knowledge}.
\newblock {MIT} Press, 1995.

\bibitem{bacchus1999171}
F.~Bacchus, J.~Y. Halpern, and H.~J. Levesque.
\newblock Reasoning about noisy sensors and effectors in the situation
  calculus.
\newblock {\em Artificial Intelligence}, 111(1--2):171 -- 208, 1999.

\bibitem{bacchus1990representing}
F.~Bacchus.
\newblock {\em Representing and Reasoning with Probabilistic Knowledge}.
\newblock MIT Press, 1990.

\bibitem{174658}
R.~Fagin and J.~Y. Halpern.
\newblock Reasoning about knowledge and probability.
\newblock {\em J. ACM}, 41(2):340--367, 1994.

\bibitem{kooi2003probabilistic}
B.P. Kooi.
\newblock {Probabilistic dynamic epistemic logic}.
\newblock {\em Journal of Logic, Language and Information}, 12(4):381--408,
  2003.

\bibitem{Kaelbling199899}
L.~P. Kaelbling, M.~L. Littman, and A.~R. Cassandra.
\newblock Planning and acting in partially observable stochastic domains.
\newblock {\em Artificial Intelligence}, 101(1--2):99 -- 134, 1998.

\bibitem{DBLP:journals/ijrr/KaelblingL13}
L.~P. Kaelbling and T.~Lozano-P{\'e}rez.
\newblock Integrated task and motion planning in belief space.
\newblock {\em I. J. Robotic Res.}, 32(9-10):1194--1227, 2013.

\bibitem{McCarthy:69}
J.~McCarthy and P.~J. Hayes.
\newblock Some philosophical problems from the standpoint of artificial
  intelligence.
\newblock In {\em Machine Intelligence}, pages 463--502, 1969.

\bibitem{reiter2001knowledge}
R.~Reiter.
\newblock {\em {Knowledge in action: logical foundations for specifying and
  implementing dynamical systems}}.
\newblock {MIT} Press, 2001.

\bibitem{DBLP:conf/ijcai/HuL11}
Y.~Hu and H.~J. Levesque.
\newblock A correctness result for reasoning about one-dimensional planning
  problems.
\newblock In {\em IJCAI}, pages 2638--2643, 2011.

\bibitem{thrun2005probabilistic}
S.~Thrun, W.~Burgard, and D.~Fox.
\newblock {\em Probabilistic Robotics}.
\newblock {MIT Press}, 2005.

\bibitem{DBLP:conf/kr/SardinaGLL06}
S.~Sardi{\~n}a, G.~{De Giacomo}, Y.~Lesp{\'e}rance, and H.~J. Levesque.
\newblock On the limits of planning over belief states under strict
  uncertainty.
\newblock In {\em KR}, pages 463--471, 2006.

\bibitem{doi:10.1177/0278364904045471}
T.~Sim{\'e}on, J.~Laumond, J.~Cort{\'e}s, and A.~Sahbani.
\newblock Manipulation planning with probabilistic roadmaps.
\newblock {\em The International Journal of Robotics Research},
  23(7-8):729--746, 2004.

\bibitem{Knepper2017}
R.~A. Knepper and M.~T. Mason.
\newblock {\em Realtime Informed Path Sampling for Motion Planning Search},
  pages 401--417.
\newblock Springer International Publishing, Cham, 2017.

\bibitem{DBLP:conf/aips/BonetG00}
B.~Bonet and H.~Geffner.
\newblock Planning with incomplete information as heuristic search in belief
  space.
\newblock In {\em AIPS}, pages 52--61, 2000.

\bibitem{petrick2004extending}
R.P.A. Petrick and F.~Bacchus.
\newblock Extending the knowledge-based approach to planning with incomplete
  information and sensing.
\newblock In {\em Proc. ICAPS}, pages 2--11, 2004.

\bibitem{Belle:2015ab}
V.~Belle and H.~J. Levesque.
\newblock Allegro: Belief-based programming in stochastic dynamical domains.
\newblock In {\em IJCAI}, 2015.

\bibitem{fikes1972learning}
R.E. Fikes, P.E. Hart, and N.J. Nilsson.
\newblock Learning and executing generalized robot plans.
\newblock {\em Artificial intelligence}, 3:251--288, 1972.

\bibitem{DBLP:conf/aips/StephanB96}
W.~Stephan and S.~Biundo.
\newblock Deduction-based refinement planning.
\newblock In {\em AIPS}, pages 213--220, 1996.

\bibitem{winner07_loop}
E.~Winner and M.~M. Veloso.
\newblock Loop{DISTILL}: Learning domain-specific planners from example plans.
\newblock In {\em Workshop on {AI} Planning and Learning, {ICAPS}}, 2007.

\bibitem{DBLP:conf/aaai/BonetPG10}
B.~Bonet, H.~Palacios, and H.~Geffner.
\newblock Automatic derivation of finite-state machines for behavior control.
\newblock In {\em AAAI}, 2010.

\bibitem{siddthesis}
S.~Srivastava.
\newblock {\em Foundations and Applications of Generalized Planning}.
\newblock PhD thesis, Department of Computer Science, University of
  Massachusetts Amherst, 2010.

\bibitem{levesque2005planning}
H.J. Levesque.
\newblock Planning with loops.
\newblock In {\em Proc. IJCAI}, pages 509--515, 2005.

\bibitem{DBLP:conf/ijcai/BonetGGR17}
B.~Bonet, G.~{De Giacomo}, H.~Geffner, and S.~Rubin.
\newblock Generalized planning: Non-deterministic abstractions and trajectory
  constraints.
\newblock In {\em IJCAI}, pages 873--879, 2017.

\bibitem{DBLP:journals/ai/LinD98}
F.~Lin and H.~J. Levesque.
\newblock What robots can do: Robot programs and effective achievability.
\newblock {\em Artif. Intell.}, 101(1-2):201--226, 1998.

\bibitem{Belle:2016ab}
V.~Belle and H.~Levesque.
\newblock Foundations for generalized planning in unbounded stochastic domains.
\newblock In {\em KR}, 2016.

\bibitem{DBLP:journals/tocl/Reiter01}
R.~Reiter.
\newblock On knowledge-based programming with sensing in the situation
  calculus.
\newblock {\em ACM Trans. Comput. Log.}, 2(4):433--457, 2001.

\bibitem{DBLP:journals/sLogica/LesperanceLLS00}
Y.~Lesp{\'e}rance, H.~J. Levesque, F.~Lin, and R.~B. Scherl.
\newblock Ability and knowing how in the situation calculus.
\newblock {\em Studia Logica}, 66(1):165--186, 2000.

\bibitem{thielscher:ICAART10}
Y.~Martin and M.~Thielscher.
\newblock Integrating reasoning about actions and {B}ayesian networks.
\newblock In {\em International Conference on Agents and Artificial
  Intelligence}, Valencia, Spain, January 2009.

\bibitem{puterman}
M.~L. Puterman.
\newblock {\em Markov Decision Processes: Discrete Stochastic Dynamic
  Programming}.
\newblock John Wiley \& Sons, Inc., New York, NY, USA, 1st edition, 1994.

\bibitem{series/synthesis/2013Geffner}
H.~Geffner and B.~Bonet.
\newblock {\em A Concise Introduction to Models and Methods for Automated
  Planning}.
\newblock Morgan and Claypool Publishers, 2013.

\bibitem{poupart2004bounded}
P.~Poupart and C.~Boutilier.
\newblock Bounded finite state controllers.
\newblock In {\em NIPS}, pages 823--830, 2004.

\bibitem{Sanner:2010fk}
S.~Sanner and K.~Kersting.
\newblock Symbolic dynamic programming for first-order pomdps.
\newblock In {\em Proc. AAAI}, pages 1140--1146, 2010.

\end{thebibliography}
\balance

\end{document}